\documentclass{article}

\usepackage{microtype}
\usepackage{graphicx}
\usepackage{subfigure}
\usepackage{booktabs} 
\usepackage{amsmath} 
\usepackage{amssymb}
\usepackage{amsthm}
\usepackage{bm}
\usepackage{setspace}

\usepackage{hyperref}
\usepackage{url}

\usepackage[accepted]{icml2021}

\renewcommand{\algorithmicrequire}{\textbf{Input:}}  
\renewcommand{\algorithmicensure}{\textbf{Output:}}
\renewcommand\algorithmiccomment[1]{%
	\hfill\#\ \eqparbox{COMMENT}{#1}%
}

\newtheorem{myprop}{Proposition}
 
\newtheorem{mytheorem}{Theorem}

\icmltitlerunning{Self-supervised Graph-level Representation Learning with Local and Global Structure}

\begin{document}

\twocolumn[
\icmltitle{Self-supervised Graph-level Representation Learning \\ with Local and Global Structure}



\icmlsetsymbol{equal}{*}

\begin{icmlauthorlist}
\icmlauthor{Minghao Xu}{sjtu}
\icmlauthor{Hang Wang}{sjtu}
\icmlauthor{Bingbing Ni}{sjtu}
\icmlauthor{Hongyu Guo}{nrc}
\icmlauthor{Jian Tang}{mila,cifar,hec}
\end{icmlauthorlist}

\icmlaffiliation{sjtu}{Shanghai Jiao Tong University}
\icmlaffiliation{nrc}{National Research Council Canada}
\icmlaffiliation{mila}{Mila - Quebec AI Institute}
\icmlaffiliation{cifar}{CIFAR AI Research Chair}
\icmlaffiliation{hec}{HEC Montr\'{e}al}

\icmlcorrespondingauthor{Minghao Xu}{xuminghao118@sjtu.edu.cn}
\icmlcorrespondingauthor{Jian Tang}{jian.tang@hec.ca}

\icmlkeywords{Self-supervised Representation Learning, Graph Representation Learning, Hierarchical Semantic Learning}

\vskip 0.3in
]



\printAffiliationsAndNotice{}  


\begin{abstract}

This paper studies unsupervised/self-supervised whole-graph representation learning, which is critical in many tasks such as molecule properties prediction in drug and material discovery. Existing methods mainly focus on preserving the local similarity structure between different graph instances but fail to discover the global semantic structure of the entire data set. In this paper, we propose a unified framework called \textbf{Lo}cal-instance and \textbf{G}lobal-semantic Learning (\emph{GraphLoG}) for self-supervised whole-graph representation learning. Specifically, besides preserving the local similarities, GraphLoG introduces the \emph{hierarchical prototypes} to capture the global semantic clusters. An efficient online expectation-maximization (EM) algorithm is further developed for learning the model. We evaluate GraphLoG by pre-training it on massive unlabeled graphs followed by fine-tuning on downstream tasks. Extensive experiments on both chemical and biological benchmark data sets demonstrate the effectiveness of the proposed approach. 

\end{abstract}


\section{Introduction} \label{sec1}

Learning informative representations of whole graphs is a fundamental problem in a variety of domains and tasks, such as molecule properties prediction in drug and material discovery~\citep{mpnn,moleculenet}, protein function forecast in biological networks~\citep{new_protein,aptrank}, and predicting the properties of circuits in circuit design~\citep{circuit-gnn}. Recently, Graph Neural Networks (GNNs) have attracted a surge of interest and showed the effectiveness in learning graph representations. These methods are usually trained in a supervised fashion, which requires a large number of labeled data. Nevertheless, in many scientific domains, labeled data are very limited and expensive to obtain. Therefore, it is becoming increasingly important to learn the representations of graphs in an unsupervised or self-supervised fashion.

Self-supervised learning has recently achieved profound success for both natural language processing, \emph{e.g.} GPT~\citep{gpt} and BERT~\citep{bert}, and image understanding, \emph{e.g.} MoCo~\citep{moco} and SimCLR~\citep{simclr}. However, how to effectively learn the representations of graphs in a self-supervised way is still an open problem. Intuitively, a desirable graph representation should be able to preserve the \emph{local-instance structure}, so that similar graphs are embedded close to each other and dissimilar ones stay far apart. In addition, the representations of the entire set of graphs should also reflect the \emph{global-semantic structure} of the data, so that the graphs with similar semantic properties are compactly embedded, which is able to benefit various downstream tasks such as graph classification or regression. 
Such global structure can be effectively captured by semantic clusters~\citep{deep_clustering,invariant_information}, which can be further organized hierarchically~\citep{prototypical_contrastive}.

There are some recent works that learn graph representation in a self-supervised manner, such as local-global mutual information maximization~\citep{graph_infomax,infograph}, structural-similarity/context prediction~\citep{pretrain_kernel,pretraining_gnn,when_does}, contrastive learning~\citep{contrastive_multi-view,gcc,contrastive_augmentation} and meta-learning~\citep{learning_to_pretrain}. However, all these methods are able to model only the local structure between different graph instances but fail to discover the global-semantic structure. To address this shortcoming, we are seeking for an approach that is sufficient to model both the local and global structure of a set of graphs. 


\begin{figure*}[t]
	\centering
	\includegraphics[width=0.86\textwidth]{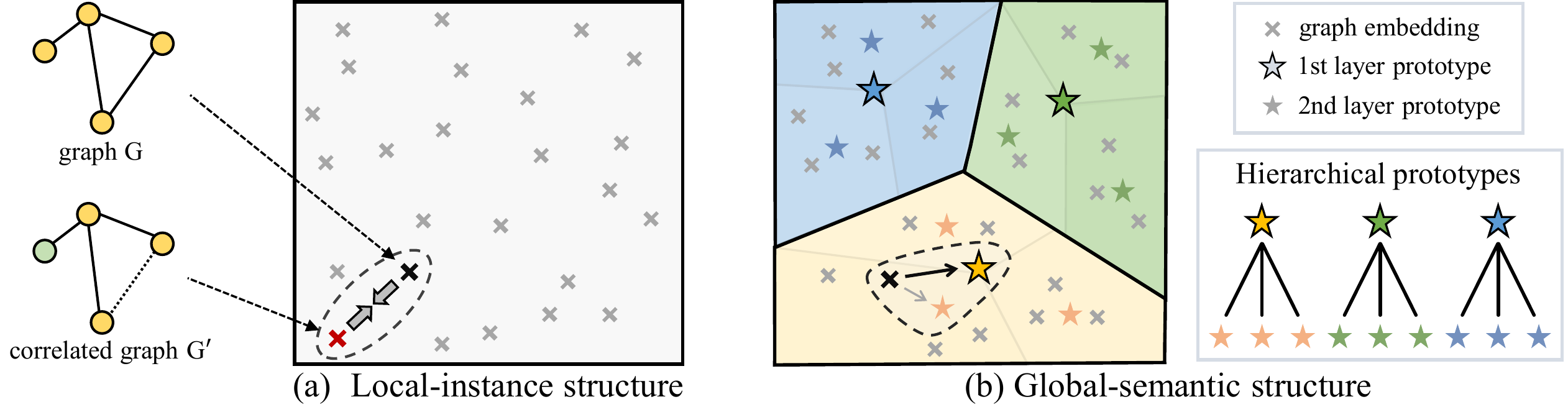}
	\vspace{-2mm}
	\caption{\textbf{Illustration of GraphLoG.} (a) Correlated graphs are constrained to be adjacently embedded to pursue the local-instance structure of the data. (b) Hierarchical prototypes are employed to discover and refine the global-semantic structure of the data.}
	\label{fig_framework}
	\vspace{-0mm}
\end{figure*}


To attain this goal, we propose a \textbf{Lo}cal-instance and \textbf{G}lobal-semantic Learning (\emph{GraphLoG}) framework for self-supervised graph representation learning. Specifically, for preserving the local similarity between various graph instances, we seek to align the embeddings of correlated graphs/subgraphs by discriminating the correlated graph/subgraph pairs from the negative pairs. In this locally smooth latent space, we further introduce the additional model parameters, \emph{hierarchical prototypes}\footnote{Hierarchical prototypes are representative cluster embeddings organized as a set of trees.}, to depict the latent distribution of a graph data set in a hierarchical way. For model learning, we propose to maximize the data likelihood with respect to both the GNN parameters and hierarchical prototypes via an online expectation-maximization (EM) algorithm. Given a mini-batch of graphs sampled from the data distribution, in the E-step, we infer the embeddings of these graphs with a GNN and sample the latent variable of each graph (\emph{i.e.} the prototypes associated to each graph) from the posterior distribution defined by current model. In the M-step, we aim to maximize the expectation of complete-data likelihood with respect to the current model by optimizing with a mini-batch-induced objective function. Therefore, in this iterative EM process, the global-semantic structure of the data can be gradually discovered and refined. The whole model is pre-trained with a large number of unlabeled graphs, and then fine-tuned and evaluated on some downstream tasks containing scarce labeled graphs. 

To verify the effectiveness of the GraphLoG framework, we apply our method to both the chemistry and biology domains. Through pre-training on massive unlabeled molecular graphs (or protein ego-networks) using the proposed local and global objectives, the existing GNN models are able to achieve superior performance on the downstream molecular property (or biological function) prediction benchmarks. In particular, the Graph Isomorphism Network (GIN)~\citep{gin} pre-trained by the proposed method outperforms the previous self-supervised graph representation learning approaches on six of eight downstream tasks of chemistry domain, and it achieves a $2.1\%$ performance gain in terms of average ROC-AUC on eight downstream tasks. In addition, the analytical experiments further illustrate the benefits of global structure learning through conducting ablation studies and visualizing the embedding distributions on a set of graphs. 


\section{Problem Definition and Preliminaries} \label{sec2}


\subsection{Problem Definition} \label{sec2_1}

An ideal representation should preserve the local structure among various data instances. More specifically, we define it as follows:

\textbf{Definition 1 (Local-instance Structure).} The local structure aims to preserve the pairwise similarity between different instances after mapping from the high-dimensional input space to the low-dimensional latent space~\citep{lle,eigenmap}. For a pair of similar graphs/subgraphs, $\mathcal{G}$ and $\mathcal{G'}$, their embeddings are expected to be nearby in the latent space, as illustrated in Fig.~\ref{fig_framework}(a), while the dissimilar pairs should be mapped to far apart.

The pursuit of local-instance structure alone is usually insufficient to capture the semantics underlying the entire data set. It is therefore important to discover the global-semantic structure of the data, which is concretely defined as follows:

\textbf{Definition 2 (Global-semantic Structure).} A real-world data set can usually be organized as various semantic clusters~\citep{semantic_structure,invariant_information}, especially in a hierarchical way for graph-structured data~\citep{gene,atc}. After mapping to the latent space, the embeddings of a set of graphs are expected to form some global structures reflecting the clustering patterns of the original data. A graphical illustration can be seen in Fig.~\ref{fig_framework}(b).

\textbf{Problem Definition.} The problem of \emph{self-supervised graph representation learning} considers a set of unlabeled graphs $\mathbf{G} = \{ \mathcal{G}_1, \mathcal{G}_2, \cdots, \mathcal{G}_{M} \}$, and we aim at learning a low-dimensional vector $h_{\mathcal{G}_m} \in \mathbb{R}^{\delta}$ for each graph $\mathcal{G}_m \in \mathbf{G}$ under the guidance of the data itself. In specific, we expect the graph embeddings $\mathbf{H} = \{ h_{\mathcal{G}_1}, h_{\mathcal{G}_2}, \cdots, h_{\mathcal{G}_M} \}$ follow both the local-instance and global-semantic structure. 


\subsection{Preliminaries} \label{sec2_2}
\textbf{Graph Neural Networks (GNNs).} A typical graph can be represented as $\mathcal{G} = (\mathcal{V}, \mathcal{E}, X_{\mathcal{V}}, X_{\mathcal{E}})$, where $\mathcal{V}$ is a set of nodes, $\mathcal{E}$ denotes the edge set, $X_{\mathcal{V}} = \{X_v | v \in \mathcal{V}\}$ stands for the attributes of all nodes, and $X_{\mathcal{E}} = \{X_{uv} | (u,v) \in \mathcal{E}\}$ represents the edge attributes. A GNN aims to learn an embedding vector $h_v \in \mathbb{R}^{\delta}$ for each node $v \in \mathcal{V}$ and also a vector $h_{\mathcal{G}} \in \mathbb{R}^{\delta}$ for the entire graph $\mathcal{G}$. For an $L$-layer GNN, a neighborhood aggregation scheme is performed to capture the $L$-hop information surrounding each node. As suggested in \citet{mpnn}, the $l$-th layer of a GNN can be formalized as follows: 
\begin{equation} \label{eq1}
\small
h_v^{(l)} = f^{(l)}_{U} \big( h_v^{(l-1)} , f^{(l)}_{M} \big( \big\{ \big( h_v^{(l-1)}, h_u^{(l-1)}, X_{uv} \big) : u \in \mathcal{N}(v) \big\} \big) \big) ,
\end{equation}
where $\mathcal{N}(v)$ is the neighborhood set of $v$, $h_v^{(l)}$ denotes the representation of node $v$ at the $l$-th layer, $h_v^{(0)}$ is initialized as the node attribute $X_v$, and $f^{(l)}_{M}$ and $f^{(l)}_{U}$ stand for the message passing and update function at the $l$-th layer, respectively. Since $h_v$ summarizes the information of a subgraph centered around node $v$, we will refer to $h_v$ as \emph{subgraph embedding} to underscore this point. The entire graph's embedding can be derived as below:
\begin{equation} \label{eq2}
h_{\mathcal{G}} = f_{R} \big( \big\{ h_{v} | v \in \mathcal{V} \big\} \big) ,
\end{equation}
where $f_{R}$ is a permutation-invariant readout function, \emph{e.g.} mean pooling or more complex graph-level pooling function~\citep{differentiable_pooling,end-to-end}.


\textbf{General EM Algorithm.} The basic objective of EM algorithm~\citep{em_incomplete,em_algorithm} is to find the maximum likelihood solution for a model containing latent variables. In such a problem, we denote the set of all observed data as $\mathbf{X}$, the set of all latent variables as $\mathbf{Z}$ and the set of all model parameters as $\bm{\theta}$. 

In the E-step, using the data $\mathbf{X}$ and the model parameters $\bm{\theta}_{t-1}$ estimated by the last EM cycle, the posterior distribution of latent variables is derived as $p(\mathbf{Z} | \mathbf{X}, \bm{\theta}_{t-1})$ which can also be regarded as the responsibility that a specific set of latent variables are taken for explaining the observations. 

In the M-step, employing the posterior distribution given by the E-step, the expectation of complete-data log-likelihood, denoted as $Q(\bm{\theta})$, is evaluated for the general model parameters $\bm{\theta}$ as follows:
\begin{equation} \label{eq3}
\begin{split}
Q(\bm{\theta}) = \mathbb{E}_{p(\mathbf{Z} | \mathbf{X}, \bm{\theta}_{t-1})} [\log p(\mathbf{X}, \mathbf{Z} | \bm{\theta})] . 
\end{split}
\end{equation}
The model parameters are updated to maximize this expectation in the M-step, which outputs:
\begin{equation} \label{eq4}
\bm{\theta}_{t} = \mathop{\arg\max}_{\bm{\theta}} \, Q(\bm{\theta}) .
\end{equation}

Each cycle of EM can increase the complete-data likelihood expected by the current model, and, considering the whole progress, the EM algorithm has been demonstrated to be capable of maximizing the marginal likelihood function $p(\mathbf{X} | \bm{\theta})$~\citep{another_interpretation,justify_em}. 


\section{GraphLoG: Self-supervised Graph-level Representation Learning with Local and Global Structure} \label{sec3} 

In this section, we introduce our approach called Local-instance and Global-semantic Learning (GraphLoG) for self-supervised graph representation learning. The main purpose of GraphLoG is to discover and refine both the local and global structures of graph embeddings in the latent space, such that we can learn useful graph representations for the downstream task like graph classification. Specifically, GraphLoG constructs a locally smooth latent space by aligning the embeddings of correlated graphs/subgraphs. On such basis, the global structures of graph embeddings are modeled by hierarchical prototypes, and the data likelihood is maximized via an online EM algorithm. Next, we elucidate the GraphLoG framework in detail. 


\subsection{Learning Local-instance Structure of Graph Representations} \label{sec3_1}

Following the existing methods for dimensionality reduction~\citep{isomap,lle,eigenmap}, the goal of local-structure learning is to preserve the local similarity of the data before and after mapping to a low-dimensional latent space. In specific, it is expected that similar graphs or subgraphs are embedded close to each other, and dissimilar ones are mapped to far apart. Using a similarity measurement defined in the latent space, we formulate this problem as maximizing the similarity of correlated graph/subgraph pairs while minimizing that of negative pairs. 

In specific, given a graph $\mathcal{G} = (\mathcal{V}, \mathcal{E}, X_{\mathcal{V}}, X_{\mathcal{E}})$ sampled from the data distribution $P_{\mathcal{G}}$, we obtain its correlated counterpart $\mathcal{G}' = (\mathcal{V}', \mathcal{E}', X_{\mathcal{V}'}, X_{\mathcal{E}'})$ through randomly masking a part of node/edge attributes in the graph~\citep{pretraining_gnn} (see appendix for the detailed scheme). In addition, for a subgraph $\mathcal{G}_{v}$ constituted by node $v$ and its $L$-hop neighborhoods in graph $\mathcal{G}$, we regard the corresponding subgraph $\mathcal{G}'_{v}$ in graph $\mathcal{G}'$ as its correlated counterpart. Through applying an $L$-layer GNN model $\mathrm{GNN}_{\theta}$ ($\theta$ stands for GNN's parameters) upon graph $\mathcal{G}$ and $\mathcal{G}'$, the graph and subgraph embeddings are derived as follows:
\begin{equation} \label{eq5}
(h_{\mathcal{V}}, h_{\mathcal{G}}) = \mathrm{GNN}_{\theta} (\mathcal{G}), \quad (h_{\mathcal{V}'}, h_{\mathcal{G}'}) = \mathrm{GNN}_{\theta} (\mathcal{G}') ,
\end{equation}
where $h_{\mathcal{V}} = \{ h_{\mathcal{G}_{v}} | v \in \mathcal{V} \}$ and $h_{\mathcal{V}'} = \{ h_{\mathcal{G}'_{v}} | v \in \mathcal{V}' \}$ represent the set of subgraph embeddings within graph $\mathcal{G}$ and $\mathcal{G}'$, respectively. 

In this phase, the objective of learning is to enhance the similarity of correlated graph/subgraph pairs while diminish that of negative pairs. Using a specific similarity measure (\emph{e.g.} cosine similarity $s(x,y) = x^{\top}y / ||x|| \ \! ||y||$ in our practice), we seeks to optimize the following two objective functions for graph and subgraph, respectively:
\begin{equation} \label{eq6}
\begin{split}
\mathcal{L}_{\textrm{graph}} = & \! - \! \mathbb{E}_{(\mathcal{G}_{+},\mathcal{G}'_{+}) \sim p(\mathbf{\mathcal{G}}, \mathbf{\mathcal{G}'}), (\mathcal{G}_{-},\mathcal{G}'_{-}) \sim p_n (\mathbf{\mathcal{G}}, \mathbf{\mathcal{G}'})} \big[ \\ 
& \qquad \qquad \quad \; \; s(\mathcal{G}_{+}, \mathcal{G}'_{+}) - s(\mathcal{G}_{-}, \mathcal{G}'_{-}) \big],
\end{split}
\end{equation}
\begin{equation} \label{eq7}
\begin{split}
\mathcal{L}_{\textrm{sub}} = & \! - \! \mathbb{E}_{(\mathcal{G}_{u},\mathcal{G}'_{u}) \sim p(\mathbf{\mathcal{G}_{v}}, \mathbf{\mathcal{G}'_{v}}), (\mathcal{G}_{v},\mathcal{G}'_{w}) \sim p_n (\mathbf{\mathcal{G}_{v}}, \mathbf{\mathcal{G}'_{v}})} \big[ \\ 
& \qquad \qquad \quad \quad \; \; \ s(\mathcal{G}_{u}, \mathcal{G}'_{u}) - s(\mathcal{G}_{v}, \mathcal{G}'_{w}) \big],
\end{split}
\end{equation}
where $p_n (\mathbf{\mathcal{G}}, \mathbf{\mathcal{G}'})$ and $p_n (\mathbf{\mathcal{G}_{v}}, \mathbf{\mathcal{G}'_{v}})$ denote the noise distribution from which negative pairs are sampled. In practice, for a correlated graph pair $(\mathcal{G}, \mathcal{G}')$ or correlated subgraph pair $(\mathcal{G}_{v}, \mathcal{G}'_{v})$, we substitute $\mathcal{G}$ ($\mathcal{G}_{v}$) randomly with another graph from the data set (a subgraph centered around another node in the same graph) to construct negative pairs. 

For learning the local-instance structure of graph representations, we aim to minimize both objective functions (Eqs.~\ref{eq6} and \ref{eq7}) with respect to the parameters of GNN:
\begin{equation} \label{eq8}
\min \limits_{\theta} \, \mathcal{L}_{\textrm{local}} , 
\end{equation}
\begin{equation} \label{eq9}
\mathcal{L}_{\textrm{local}} = \mathcal{L}_{\textrm{graph}} + \mathcal{L}_{\textrm{sub}} .
\end{equation}


\subsection{Learning Global-semantic Structure of Graph Representations} \label{sec3_2}

It is worth noticing that the graphs in a data set may possess hierarchical semantic information. For example, drugs (\emph{i.e.} molecular graphs) are represented by a five-level hierarchy in the Anatomical Therapeutic Chemical (ATC) classification system~\citep{atc}. After mapping to the latent space, the embeddings of all graphs in the data set are also expected to form some global structures corresponding to the hierarchical semantic structures of the original data.

However, for the lack of explicit semantic labels in the self-supervised graph representation learning, such global structure cannot be attained via label-induced supervision. To tackle this limitation, we introduce an additional set of model parameters, \emph{hierarchical prototypes}, to represent the feature clusters in the latent space in a hierarchical way. They are formally defined as $\mathbf{C} = \{c^{l}_{i}\}_{i=1}^{M_{l}}$ ($l = 1, 2, \cdots , L_p$), where $c^{l}_{i} \in \mathbb{R}^{\delta}$ stands for the $i$-th prototype at the $l$-th layer, $L_p$ is the depth of hierarchical prototypes, and $M_{l}$ denotes the number of prototypes at the $l$-th layer. These prototypes are structured as a set of trees (Fig.~\ref{fig_framework}(b)), in which each node corresponds to a prototype, and, except for the leaf nodes, each prototype possesses a set of child nodes, denoted as $\mathbb{C}(c^{l}_{i})$ ($1 \leqslant i \leqslant M_{l}$, $l = 1, 2, \cdots , L_p - 1$). 

The goal of global-semantic learning is to encourage the graphs to be compactly embedded around corresponding prototypes and, at the same time, refine hierarchical prototypes to better represent the data. We formalize this problem as optimizing a latent variable model. Specifically, for the observed data set $\mathbf{G} = \{ \mathcal{G}_1, \mathcal{G}_2, \cdots, \mathcal{G}_{M} \}$, we consider a latent variable set, \emph{i.e.} the prototype assignments of all graphs $\mathbf{Z} = \{ z_{\mathcal{G}_1}, z_{\mathcal{G}_2}, \cdots, z_{\mathcal{G}_M} \}$ ($z_{\mathcal{G}_m}$ is a set of prototypes that best represent $\mathcal{G}_m$ in the latent space). The model parameters in this problem are the GNN parameters $\theta$ and hierarchical prototypes $\mathbf{C}$. Since the corresponding latent variable of each graph is not given, it is hard to directly maximize the complete-data likelihood function $p(\mathbf{G}, \mathbf{Z} | \theta, \mathbf{C})$. Therefore, we seek to maximize the expectation of complete-data likelihood through the EM algorithm. 

The vanilla EM algorithm~\citep{em_incomplete,em_algorithm} requires a full pass through the data set before each parameter update, which is computationally inefficient when the size of data set is large like in our case. Therefore, we consider an online EM variant~\citep{online_em_1,online_em_2,online_em_3} which operates on mini-batches of data. This approach is based on the \emph{i.i.d.} assumption of the data set, where both the complete-data likelihood and the posterior probability of latent variables can be factorized over each observed-latent variable pair:
\begin{equation} \label{eq10}
p(\mathbf{G}, \mathbf{Z} | \theta, \mathbf{C}) = \prod_{m=1}^{M} p(\mathcal{G}_m, z_{\mathcal{G}_m} | \theta, \mathbf{C}) ,
\end{equation}
\begin{equation} \label{eq11}
p(\mathbf{Z} | \mathbf{G}, \theta, \mathbf{C}) = \prod_{m=1}^{M} p(z_{\mathcal{G}_m} | \mathcal{G}_m, \theta, \mathbf{C}) .
\end{equation}
First, we introduce the initialization scheme of model parameters. 


\textbf{Initialization of model parameters.} Before triggering the global structure exploration, we first pre-train the GNN by minimizing $\mathcal{L}_{\mathrm{local}}$ and employ the derived GNN model as initialization, which establishes a locally smooth latent space. After that, we utilize this pre-trained GNN model to extract the embeddings of all graphs in the data set, and the K-means clustering is applied upon these graph embeddings to initialize the bottom layer prototypes (\emph{i.e.} $\{c^{L_p}_{i}\}_{i=1}^{M_{L_p}}$) with the output cluster centers. The prototypes of upper layers are initialized by iteratively applying K-means clustering to the prototypes of the layer below. For each time of clustering, we discard the output cluster centers assigned with less than two samples to avoid trivial solutions~\citep{diffrac,deep_clustering}. 

Next, we state the details of the E-step and M-step applied in our method.  


\textbf{E-step.} In this step, we first randomly sample a mini-batch of graphs $\widetilde{\mathbf{G}} = \{ \mathcal{G}_1, \mathcal{G}_2, \cdots, \mathcal{G}_{N} \}$ ($N$ denotes the batch size) from the data set $\mathbf{G}$, and $\widetilde{\mathbf{Z}} = \{ z_{\mathcal{G}_n} \}_{n=1}^{N}$ stands for the latent variables corresponding to these sampled graphs. Each latent variable $z_{\mathcal{G}_n} = \{ z^{1}_{\mathcal{G}_n}, z^{2}_{\mathcal{G}_n}, \cdots , z^{L_p}_{\mathcal{G}_n} \}$ is a chain of prototypes from top layer to bottom layer that best represent graph $\mathcal{G}_n$ in the latent space, and it holds that $z^{l+1}_{\mathcal{G}_n}$ is the child node of $z^{l}_{\mathcal{G}_n}$ in the corresponding tree structure, \emph{i.e.} $z^{l+1}_{\mathcal{G}_n} \in \mathbb{C}(z^{l}_{\mathcal{G}_n})$ ($l = 1, 2, \cdots , L_p - 1$). The posterior distribution of $\widetilde{\mathbf{Z}}$ can be evaluated in a factorized way because of the \emph{i.i.d.} assumption:
\begin{equation} \label{eq12}
p(\widetilde{\mathbf{Z}} | \widetilde{\mathbf{G}}, \theta_{t-1}, \mathbf{C}_{t-1}) = \prod_{n=1}^{N} p(z_{\mathcal{G}_n} | \mathcal{G}_n, \theta_{t-1}, \mathbf{C}_{t-1}) ,
\end{equation}
where $\theta_{t-1}$ and $\mathbf{C}_{t-1}$ are the model parameters from the last EM cycle. Directly evaluating each posterior distribution $p(z_{\mathcal{G}_n} | \mathcal{G}_n, \theta_{t-1}, \mathbf{C}_{t-1})$ is nontrivial, which requires to traverse all the possible chains in hierarchical prototype. Instead, we adopt the idea of stochastic EM algorithm~\citep{two_stochastic,stochastic_EM} and draw a sample $\hat{z}_{\mathcal{G}_n} \sim p(z_{\mathcal{G}_n} | \mathcal{G}_n, \theta_{t-1}, \mathbf{C}_{t-1})$ for Monte Carlo estimation. In specific, we sequentially sample a prototype from each layer in a top-down manner, and all the sampled prototypes form a connected chain from top layer to bottom layer in hierarchical prototypes. Formally, we first sample a prototype from top layer according to a categorical distribution over all the top layer prototypes, \emph{i.e.} $\hat{z}^{1}_{\mathcal{G}_n} \sim \mathrm{Cat}(z^{1}_{\mathcal{G}_n} | \{ \alpha_{i} \}_{i=1}^{M_{1}})$ ($\alpha_i = \mathrm{softmax}(s(c^{1}_{i}, h_{\mathcal{G}_n}))$), where $s$ denotes the cosine similarity measurement; for the sampling at layer $l$ ($l \geqslant 2$), we draw a prototype from that layer based on a categorical distribution over the child nodes of prototype $\hat{z}^{l-1}_{\mathcal{G}_n}$ sampled from the layer above, \emph{i.e.} $\hat{z}^{l}_{\mathcal{G}_n} \sim \mathrm{Cat}(z^{l}_{\mathcal{G}_n} | \{ \alpha_{c} \})$ ($\alpha_c = \mathrm{softmax}(s(c, h_{\mathcal{G}_n}))$, $\forall c \in \mathbb{C}(\hat{z}^{l-1}_{\mathcal{G}_n})$), such that we sample a latent variable $ \hat{z}_{\mathcal{G}_n} = \{\hat{z}^{1}_{\mathcal{G}_n}, \hat{z}^{2}_{\mathcal{G}_n}, \cdots , \hat{z}^{L_p}_{\mathcal{G}_n}\}$ which is a connected chain in hierarchical prototypes. Using the latent variables inferred as above, we seek to maximize the expectation of complete-data log-likelihood in the M-step. 


\textbf{M-step.} In this step, we aim at maximizing the expected complete-data log-likelihood with respect to the posterior distribution of latent variables, which is defined as follows:
\begin{equation} \label{eq13}
Q(\theta, \mathbf{C}) = \mathbb{E}_{p(\mathbf{Z} | \mathbf{G}, \theta_{t-1}, \mathbf{C}_{t-1})} [\log p(\mathbf{G}, \mathbf{Z} | \theta, \mathbf{C})] . 
\end{equation}
This expectation needs the computation over all data points, which cannot be attained in the online setting. As a substitute, we propose to maximize the expected log-likelihood on mini-batch $\widetilde{\mathbf{G}}$, which can be estimated using the latent variables $\widetilde{\mathbf{Z}}_{est} = \{\hat{z}_{\mathcal{G}_n}\}_{n=1}^{N}$ sampled in the E-step:
\begin{equation} \label{eq14}
\begin{split}
\widetilde{Q}(\theta, \mathbf{C}) & = \mathbb{E}_{p(\widetilde{\mathbf{Z}} | \widetilde{\mathbf{G}}, \theta_{t-1}, \mathbf{C}_{t-1})} [ \log p(\widetilde{\mathbf{G}}, \widetilde{\mathbf{Z}} | \theta, \mathbf{C}) ] \\
& \approx \log p(\widetilde{\mathbf{G}}, \widetilde{\mathbf{Z}}_{est} | \theta, \mathbf{C}) \\
& = \sum_{n=1}^{N} \log p(\mathcal{G}_n, \hat{z}_{\mathcal{G}_n} | \theta, \mathbf{C}) .
\end{split}
\end{equation}
We would like to point out that $\widetilde{Q}(\theta, \mathbf{C})$ is a decent proxy for $Q(\theta, \mathbf{C})$, where a proportional relation approximately holds between them (see appendix for the proof):
\begin{equation} \label{eq15}
\widetilde{Q}(\theta, \mathbf{C}) \approx \frac{N}{M} \, Q(\theta, \mathbf{C}) .
\end{equation}
We further scale $\widetilde{Q}(\theta, \mathbf{C})$ with the batch size to derive the log-likelihood function $\mathcal{L}(\theta, \mathbf{C})$ that is more stable in terms of computation:
\begin{equation} \label{eq16}
\mathcal{L}(\theta, \mathbf{C}) = \frac{1}{N} \, \widetilde{Q}(\theta, \mathbf{C}). 
\end{equation}
To estimate $\mathcal{L}(\theta, \mathbf{C})$, we need to define the joint likelihood of a graph $\mathcal{G}$ and a latent variable $z_{\mathcal{G}}$, which is represented with an energy-based formulation in our method:
\begin{equation} \label{eq17}
p(\mathcal{G}, z_{\mathcal{G}} | \theta, \mathbf{C}) = \frac{1}{Z(\theta, \mathbf{C})} \exp \big( f(h_{\mathcal{G}}, z_{\mathcal{G}}) \big) ,
\end{equation}
where $Z(\theta, \mathbf{C})$ denotes the partition function. We formalize the negative energy function $f$ by measuring the similarities between graph embedding $h_{\mathcal{G}}$ and the prototypes in $z_{\mathcal{G}}$ and also measuring the similarities between the prototypes in $z_{\mathcal{G}}$ that are from consecutive layers:
\begin{equation} \label{eq18}
\small
f(h_{\mathcal{G}}, z_{\mathcal{G}}) = \sum_{l=1}^{L_p} s \big( h_{\mathcal{G}}, z^{l}_{\mathcal{G}} \big) + \sum_{l=1}^{L_p - 1} s \big( z^{l}_{\mathcal{G}}, z^{l+1}_{\mathcal{G}} \big) .
\end{equation}
Intuitively, $f$ evaluates how well a latent variable $z_{\mathcal{G}}$ represents graph $\mathcal{G}$ in the latent space, and it also measures the affinity between the consecutive prototypes along a chain from top layer to bottom layer in hierarchical prototypes.  

It is nontrivial to optimize with $p(\mathcal{G}, z_{\mathcal{G}} | \theta, \mathbf{C})$ due to the intractable partition function. Inspired by Noise-Contrastive Estimation (NCE)~\citep{NCE_1,NCE_2}, 
we seek to optimize with the unnormalized likelihoods, \emph{i.e.} $\tilde{p}(\mathcal{G}, z_{\mathcal{G}} | \theta, \mathbf{C}) = \exp ( f(h_{\mathcal{G}}, z_{\mathcal{G}}) )$, by contrasting the positive observed-latent variable pair with the negative pairs sampled from some noise distribution, which defines an objective function that well approximates $\mathcal{L}(\theta, \mathbf{C})$:
\begin{equation} \label{eq19}
\small
\begin{split}
\mathcal{L}_{\textrm{global}} & \! = \! - \mathbb{E}_{(\mathcal{G}^{+}, z^{+}_{\mathcal{G}}) \sim p(\mathcal{G}, z_{\mathcal{G}})} \Big\{ \log \tilde{p}(\mathcal{G}^{+}, z^{+}_{\mathcal{G}} | \theta, \mathbf{C}) \\
& \quad - \mathbb{E}_{(\mathcal{G}^{-}, z^{-}_{\mathcal{G}}) \sim p_n (\mathcal{G}, z_{\mathcal{G}})} \big[ \log \tilde{p}(\mathcal{G}^{-}, z^{-}_{\mathcal{G}} | \theta, \mathbf{C}) \big] \Big\} ,
\end{split}
\end{equation}
where $p_n (\mathcal{G}, z_{\mathcal{G}})$ is the noise distribution. 
In practice, we compute the outer expectation with all the positive pairs in the mini-batch, \emph{i.e.} $(\mathcal{G}_n, \hat{z}_{\mathcal{G}_n})$ ($1 \leqslant n \leqslant N$), and, for computing the inner expectation, we construct $L_p$ negative pairs for the positive pair $(\mathcal{G}_n, \hat{z}_{\mathcal{G}_n})$ by fixing the graph $\mathcal{G}_n$ and randomly substituting one of $L_p$ prototypes in $\hat{z}_{\mathcal{G}_n}$ with another prototype at the same layer each time. For global-semantic learning, we aim to minimize the global objective function $\mathcal{L}_{\textrm{global}}$ with respect to both the GNN parameter $\theta$ and hierarchical prototypes $\mathbf{C}$:
\begin{equation} \label{eq20}
\min \limits_{\theta, \mathbf{C}} \, \mathcal{L}_{\textrm{global}} .
\end{equation}

In general, the proposed online EM algorithm seeks to maximize the joint likelihood $p(\mathbf{G}, \mathbf{Z} | \theta, \mathbf{C})$ governed by model parameters $\theta$ and $\mathbf{C}$. For a step further, we propose the following proposition that this algorithm can indeed maximize the marginal likelihood function $p(\mathbf{G} | \theta, \mathbf{C})$.
\begin{myprop} \label{prop1}
For each EM cycle, the model parameters $\theta$ and $\mathbf{C}$ are updated in such a way that increases the marginal likelihood function $p(\mathbf{G} | \theta, \mathbf{C})$, unless a local maximum is reached on the mini-batch log-likelihood function $\widetilde{Q}(\theta, \mathbf{C})$.
\end{myprop}
The proof of Proposition~\ref{prop1} is provided in the appendix.


\begin{algorithm}[tb] 
   \caption{Optimization Algorithm of GraphLoG.}
   \label{algo1}
\begin{spacing}{1.05}
\begin{algorithmic}
   \STATE {\bfseries Input:} Unlabeled graph data set $\mathbf{G}$, the number of \\ learning steps $T$.
   \STATE {\bfseries Output:} Pre-trained GNN model $\mathrm{GNN}_{\theta_{T}}$.
   \STATE Pre-train GNN with local objective function (Eq.~\ref{eq9}).
   \STATE Initialize model parameters $\theta_0$ and $\mathbf{C}_0$.
   \FOR{$t=1$ {\bfseries to} $T$}
   \STATE Sample a mini-batch $\widetilde{\mathbf{G}}$ from $\mathbf{G}$.
   \STATE $\Diamond$ \emph{E-step}:
   \STATE Sample latent variables $\widetilde{\mathbf{Z}}_{est}$ with $\mathrm{GNN}_{\theta_{t-1}}$ and $\mathbf{C}_{t-1}$.
   \STATE $\Diamond$ \emph{M-step}:
   \STATE Update model parameters:
   \STATE $\quad \ \, \theta_t \gets \theta_{t-1} - \nabla_{\theta} (\mathcal{L}_{\textrm{local}} + \mathcal{L}_{\textrm{global}})$,
   \STATE $\quad \ \, \mathbf{C}_t \gets \mathbf{C}_{t-1} - \nabla_{\mathbf{C}} (\mathcal{L}_{\textrm{local}} + \mathcal{L}_{\textrm{global}})$.
   \ENDFOR
\end{algorithmic}
\end{spacing}
\end{algorithm}


\subsection{Model Optimization and Downstream Application} \label{sec3_3}

The GraphLoG framework seeks to learn the graph representations preserving both the local-instance and global-semantic structure on an unlabeled graph data set $\mathbf{G}$. For model optimization under this framework, we first pre-train the GNN by minimizing the local objective function $\mathcal{L}_{\textrm{local}}$ and initialize the model parameters with the pre-trained GNN. After that, for each learning step, we sample a mini-batch $\widetilde{\mathbf{G}}$ from the data set and conduct an EM cycle. In the E-step, the latent variables corresponding to the mini-batch, \emph{i.e.} $\widetilde{\mathbf{Z}}_{est}$, are sampled from the posterior distribution defined by the current model. In the M-step, model parameters are updated to maximize the expected log-likelihood on mini-batch $\widetilde{\mathbf{G}}$. Also, we add the local objective function to the optimization of M-step, which guarantees the local smoothness when pursuing the global-semantic structure and performs well in practice. We summarize the optimization algorithm in Alg.~\ref{algo1}. 

After the self-supervised pre-training on massive unlabeled graphs, the derived GNN model can be applied to various downstream tasks for producing effective embedding vectors of different graphs. For example, we can first pre-train a GNN model with GraphLoG on a large number of unlabeled molecules (\emph{i.e.} molecular graphs). After that, for a downstream task of molecular property prediction where a small set of labeled molecules are available, we can learn a linear classifier upon the pre-trained GNN model to perform the specific graph classification task. 


\section{Related Work} \label{sec4} 

\vspace{-0.5mm}
\textbf{Graph Neural Networks (GNNs).} Recently, following the efforts of learning graph representations via optimizing random walk~\citep{deepwalk,line,node2vec,graph2vec} or matrix factorization~\citep{grarep,sdne} objectives, GNNs explicitly derive proximity-preserved feature vectors in a neighborhood aggregation way. As suggested in~\citet{mpnn}, the forward pass of most GNNs can be depicted in two phases, \emph{Message Passing} and \emph{Readout} phase, and various works~\citep{molecular_fingerprint,gcn,graphsage,gat,differentiable_pooling,end-to-end,gin} sought to improve the effectiveness of these two phases. Unlike these methods which are mainly trained in a supervised fashion, our approach aims for self-supervised learning for GNNs.


\textbf{Self-supervised Learning for GNNs.} There are some recent works that explored self-supervised graph representation learning with GNNs. \citet{embedding_propagation} learned graph representations by embedding propagation, and \citet{graph_infomax}, \citet{infograph} achieved this goal through mutual information maximization. Also, some self-supervised tasks, \emph{e.g.} edge prediction~\citep{gvae}, context prediction~\citep{pretraining_gnn,graph_transformer}, graph partitioning~\citep{when_does}, edge/attribute generation~\citep{gpt-gnn} and contrastive learning~\citep{contrastive_multi-view,gcc,contrastive_augmentation}, have been designed to acquire knowledge from unlabeled graphs. Nevertheless, all these methods are only able to model the local relations between different graph instances. The proposed framework seeks to discover both the local-instance and global-semantic structure of a set of graphs.


\begin{table*}[t]
	\begin{spacing}{1.15}
		\centering
		\scriptsize
		\caption{Test ROC-AUC (\%) on downstream molecular property prediction benchmarks.} \label{tab_mol}
		\setlength{\tabcolsep}{1.7mm}
		\begin{tabular}{c|cccccccc|c} 
			\toprule[1.0pt]
			Methods & BBBP & Tox21 & ToxCast & SIDER & ClinTox & MUV & HIV & BACE & Avg \\
			\hline
			\hline
			Random & $65.8\pm4.5$ & $74.0\pm0.8$ & $63.4\pm 0.6$ & $57.3\pm1.6$ & $58.0\pm4.4$ & $71.8\pm2.5$ & $75.3\pm1.9$ & $70.1\pm5.4$ & $67.0$ \\
			\hline
			EdgePred (\citeyear{gvae}) & $67.3\pm2.4$ & $76.0\pm0.6$ & $64.1\pm0.6$ & $60.4\pm0.7$ & $64.1\pm3.7$ & $74.1\pm2.1$ & $76.3\pm1.0$ & $79.9\pm0.9$ & $70.3$ \\
			InfoGraph (\citeyear{infograph}) & $68.2\pm0.7$ & $75.5\pm0.6$ & $63.1\pm0.3$ & $59.4\pm1.0$ & $70.5\pm1.8$ & $75.6\pm1.2$ & $77.6\pm0.4$ & $78.9\pm1.1$ & $71.1$ \\
			AttrMasking (\citeyear{pretraining_gnn}) & $64.3\pm2.8$ & $\mathbf{76.7}\pm0.4$ & $\mathbf{64.2}\pm0.5$ & $61.0\pm0.7$ & $71.8\pm4.1$ & $74.7\pm1.4$ & $77.2\pm1.1$ & $79.3\pm1.6$ & $71.1$ \\
			ContextPred (\citeyear{pretraining_gnn}) & $68.0\pm2.0$ & $75.7\pm0.7$ & $63.9\pm0.6$ & $60.9\pm0.6$ & $65.9\pm3.8$ & $75.8\pm1.7$ & $77.3\pm1.0$ & $79.6\pm1.2$ & $70.9$ \\
			GraphPartition (\citeyear{when_does}) & $70.3\pm0.7$ & $75.2\pm0.4$ & $63.2\pm0.3$ & $61.0\pm0.8$ & $64.2\pm0.5$ & $75.4\pm1.7$ & $77.1\pm0.7$ & $79.6\pm1.8$ & $70.8$ \\
			GraphCL (\citeyear{contrastive_augmentation}) & $69.5\pm0.5$ & $75.4\pm0.9$ & $63.8\pm0.4$ & $60.8\pm0.7$ & $70.1\pm1.9$ & $74.5\pm1.3$ & $77.6\pm0.9$ & $78.2\pm1.2$ & $71.3$ \\
			\hline
			GraphLoG (ours) & $\mathbf{72.5}\pm0.8$ & $75.7\pm0.5$ & $63.5\pm0.7$ & $\mathbf{61.2}\pm1.1$ & $\mathbf{76.7}\pm3.3$ & $\mathbf{76.0}\pm1.1$ & $\mathbf{77.8}\pm0.8$ & $\mathbf{83.5}\pm1.2$ & $\mathbf{73.4}$ \\
			\bottomrule[1.0pt]
		\end{tabular}
	\end{spacing}
\end{table*}


\begin{table}[t]
	\begin{spacing}{1.15}
		\centering
		\scriptsize
		\vspace{-6mm}
		\caption{Test ROC-AUC (\%) on downstream biological function prediction benchmark.} \label{tab_bio}
		\setlength{\tabcolsep}{6.8mm}
		\begin{tabular}{c|c}
			\toprule[1.0pt]
			Methods & ROC-AUC (\%) \\
			\hline
			\hline
			Random & $64.8\pm1.0$ \\
			\hline
			EdgePred~\citep{gvae} & $70.5\pm0.7$  \\
			InfoGraph~\citep{infograph} & $70.7\pm0.5$ \\
			AttrMasking~\citep{pretraining_gnn} & $70.5\pm0.5$ \\
			ContextPred~\citep{pretraining_gnn} & $69.9\pm0.3$ \\
			GraphPartition~\citep{when_does} & $71.0\pm0.2$ \\
			GraphCL~\citep{contrastive_augmentation} & $71.2\pm0.6$ \\
			\hline
			GraphLoG (ours) & $\mathbf{72.9}\pm0.7$ \\
			\bottomrule[1.0pt]		
		\end{tabular}
		\vspace{-1.5mm}
	\end{spacing}
\end{table}


\textbf{Self-supervised Semantic Learning.} Clustering-based methods~\citep{unsupervised_deep_embedding,joint_cluster,k-means-friendly,deep_clustering,invariant_information,prototypical_contrastive} are commonly used to learn the semantic information of the data in a self-supervised fashion. Among which, DeepCluster~\citep{deep_clustering} proved the strong transferability of the visual representations learnt by clustering prediction to various downstream visual tasks. Prototypical Contrastive Learning~\citep{prototypical_contrastive} proved its superiority over the instance-level contrastive learning approaches.
These methods are mainly developed for images but not for graph-structured data. Furthermore, the hierarchical semantic structure of the data has been less explored in previous works.


\vspace{-1mm}
\section{Experiments} \label{sec5}
\vspace{-0.5mm}

In this section, we evaluate the performance of GraphLoG on both the chemistry and biology domains using the procedure of pre-training followed by fine-tuning. Also, analytical studies are conducted to verify the effectiveness of local and global structure learning. 


\subsection{Experimental Setup} \label{sec5_1}

\textbf{Pre-training details.} Following \citet{pretraining_gnn}, we adopt a five-layer Graph Isomorphism Network (GIN)~\citep{gin} with 300-dimensional hidden units and a mean pooling readout function for performance comparisons (Secs.~\ref{sec5_2} and \ref{sec5_3}). We use an Adam optimizer~\citep{adam} (learning rate: $1 \times 10^{-3}$) to pre-train the GNN with $\mathcal{L}_{\textrm{local}}$ for one epoch and then train the whole model with both $\mathcal{L}_{\textrm{local}}$ and $\mathcal{L}_{\textrm{global}}$ for 10 epochs. For each time of K-means clustering in the initialization of hierarchical prototypes, we adopt 50 initial cluster centers. Unless otherwise specified, the batch size $N$ is set as 512, and the hierarchical prototypes' depth $L_p$ is set as 3. These hyperparameters are selected by the grid search on the validation sets of four downstream molecule data sets (\emph{i.e.} BBBP, SIDER, ClinTox and BACE), and their sensitivity is analyzed in Sec.~\ref{sec5_4}. 


\textbf{Fine-tuning details.} For fine-tuning on a downstream task, a linear classifier is appended upon the pre-trained GNN, and an Adam optimizer (learning rate: $1 \times 10^{-3}$, fine-tuning batch size: 32) is employed to train the model for 100 epochs. We utilize a learning rate scheduler with fix step size which multiplies the learning rate by 0.3 every 30 epochs. All the reported results are averaged over five independent runs. The source code is available at \url{https://github.com/DeepGraphLearning/GraphLoG}.


\textbf{Performance comparison.} For the experiments on both chemistry and biology domains, we compare the proposed method with existing self-supervised graph representation learning algorithms (\emph{i.e.} EdgePred~\citep{gvae}, InfoGraph~\citep{infograph}, AttrMasking~\citep{pretraining_gnn}, ContextPred~\citep{pretraining_gnn}, GraphPartition~\citep{when_does} and GraphCL~\citep{contrastive_augmentation}) to verify its effectiveness. We report the results of EdgePred, AttrMasking and ContextPred from \citet{pretraining_gnn} and examine the performance of InfoGraph, GraphPartition and GraphCL based on the released source code. 


\begin{table}[t]
	\begin{spacing}{1.47}
		\centering
		\scriptsize
		\vspace{-6mm}
		\caption{Test ROC-AUC (\%) of different methods under four GNN architectures. (All results are reported on biology domain.)} \label{tab_gnn}
		\setlength{\tabcolsep}{0.9mm}
		\begin{tabular}{c|cccc}
			\toprule[1.0pt]
			Methods & GCN & GraphSAGE & GAT & GIN \\
			\hline
			\hline
			Random & $63.2\pm1.0$ & $65.7\pm1.2$ & $68.2\pm1.1$ & $64.8\pm1.0$\\
			\hline
			EdgePred (\citeyear{gvae}) & $68.0\pm0.9$ & $67.8\pm0.7$ & $67.9\pm1.3$ & $70.5\pm0.7$ \\
			AttrMasking (\citeyear{pretraining_gnn}) & $68.3\pm0.8$ & $69.2\pm0.6$ & $67.3\pm0.8$ & $70.5\pm0.5$ \\
			ContextPred (\citeyear{pretraining_gnn}) & $67.6\pm0.3$ & $69.6\pm0.6$ & $66.9\pm1.2$ & $69.9\pm0.3$ \\
			GraphCL (\citeyear{contrastive_augmentation}) & $69.1\pm0.9$ & $70.2\pm0.4$ & $68.4\pm1.2$ & $71.2\pm0.6$ \\
			\hline
			GraphLoG (ours) & $\mathbf{71.2}\pm0.6$ & $\mathbf{70.8}\pm0.8$ & $\mathbf{69.5}\pm1.0$ & $\mathbf{72.9}\pm0.7$ \\
			\bottomrule[1.0pt]		
		\end{tabular}
		\vspace{-1.5mm}
	\end{spacing}
\end{table}


\begin{figure*}[t]
	\centering
	\includegraphics[width=0.9\textwidth]{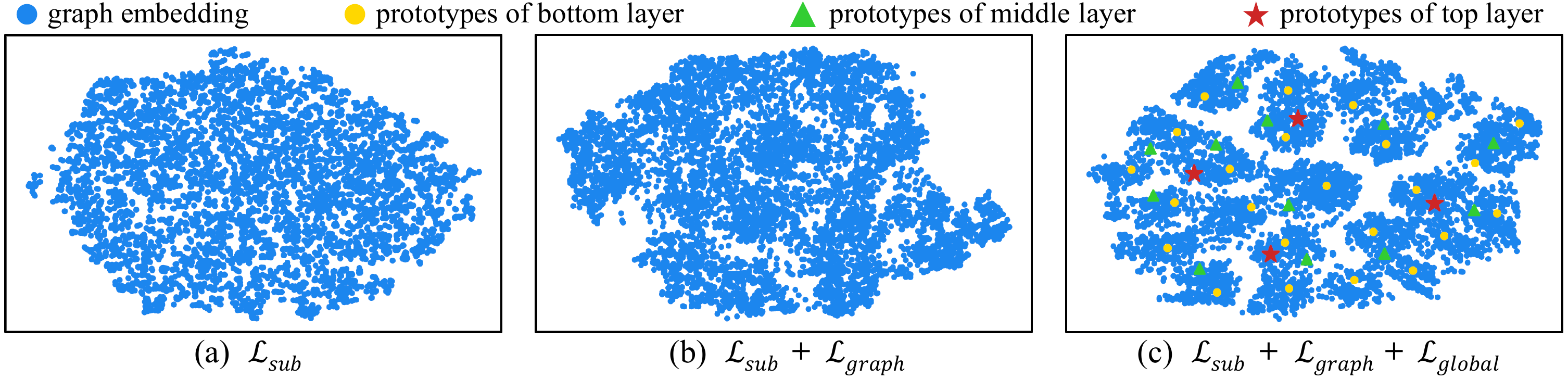}
	\vspace{-3.5mm}
	\caption{The t-SNE visualization on ZINC15 database (\emph{i.e.} the pre-training data set for chemistry domain).} 
	\label{fig_tsne}
\end{figure*}


\subsection{Experiments on Chemistry Domain} \label{sec5_2}

\textbf{Data sets.} For fair comparison, we use the same data sets as in \citet{pretraining_gnn}. In specific, a subset of ZINC15 database~\citep{zinc15} with 2 million unlabeled molecules is employed for self-supervised pre-training. Eight binary classification data sets in MoleculeNet~\citep{moleculenet} serve as downstream tasks, where the scaffold split scheme~\citep{scaffold_split} is used for data set split. 


\textbf{Results.} In Tab.~\ref{tab_mol}, we report the performance of the proposed GraphLoG method compared with other works, where `Random' denotes the GIN model with random initialization. Among all self-supervised learning strategies, our approach achieves the best performance on six of eight tasks, and a $2.1\%$ performance gain is obtained in terms of average ROC-AUC. We deem that this improvement over previous works is mainly from the global structure modeling in GraphLoG, which is not included in existing methods. 


\subsection{Experiments on Biology Domain} \label{sec5_3}

\textbf{Data sets.} For biology domain, following \citet{pretraining_gnn}, 395K unlabeled protein ego-networks are utilized for self-supervised pre-training. The downstream task is to predict 40 fine-grained biological functions of 8 species. 


\textbf{Results.} Tab.~\ref{tab_bio} reports the test ROC-AUC of various self-supervised learning techniques. It can be observed that the proposed GraphLoG method outperforms existing approaches with a clear margin, \emph{i.e.} a $1.7\%$ performance gain. This result illustrates that the proposed approach is able to learn effective graph representations that benefit the downstream task involving fine-grained classification.

In Tab.~\ref{tab_gnn}, we further compare GraphLoG with four existing methods under four GNN architectures (\emph{i.e.} GCN~\citep{gcn}, GraphSAGE~\citep{graphsage}, GAT~\citep{gat} and GIN~\citep{gin}). We can observe that GraphLoG outperforms the existing approaches on all configurations, and, compared to EdgePred, AttrMasking and ContextPred, it avoids the performance decrease relative to random initialization baseline on GAT. 


\begin{table}[t]
	\begin{spacing}{1.15}
		\centering
		\scriptsize
		\vspace{-7mm}
		\caption{Ablation study for different objective functions on downstream biological function prediction benchmark.} \label{tab_ablation}
		\vspace{0.5mm}
		\setlength{\tabcolsep}{5.5mm}
		\begin{tabular}{ccc|c}
			\toprule[1.0pt]
			$\mathcal{L}_{\textrm{sub}}$ & $\mathcal{L}_{\textrm{graph}}$ & $\mathcal{L}_{\textrm{global}}$ & ROC-AUC (\%) \\
			\hline
			\hline
			$\checkmark$ &  &   & $70.1\pm0.6$ \\
			& $\checkmark$  &  & $71.0\pm0.3$ \\
			&  & $\checkmark$ & $71.5\pm0.5$ \\
			$\checkmark$ & $\checkmark$ &  & $71.3\pm0.7$ \\
			$\checkmark$ &  & $\checkmark$ & $71.9\pm0.8$ \\
			& $\checkmark$ & $\checkmark$ & $72.2\pm0.4$ \\
			$\checkmark$ & $\checkmark$ & $\checkmark$ & $\mathbf{72.9}\pm0.7$ \\
			\bottomrule[1.0pt]
		\end{tabular}
	\end{spacing}
	\vspace{-1.5mm}
\end{table}


\subsection{Analysis} \label{sec5_4}

\textbf{Effect of different objective functions.} In Tab.~\ref{tab_ablation}, we analyze the effect of three objective functions on the biology domain, and we continue using the GIN depicted in Sec.~\ref{sec5_1} in this experiment. When each objective function is individually applied (1st, 2nd and 3rd row), the one for global-semantic learning performs best, which probably benefits from its exploration of the semantic structure of the data. Through simultaneously applying different objective functions, the full model (last row) achieves the best performance, which illustrates that the learning of local and global structure are complementary to each other. 


\begin{figure}[t]
	\centering
	\vspace{-3.5mm}
	\includegraphics[width=0.45\textwidth]{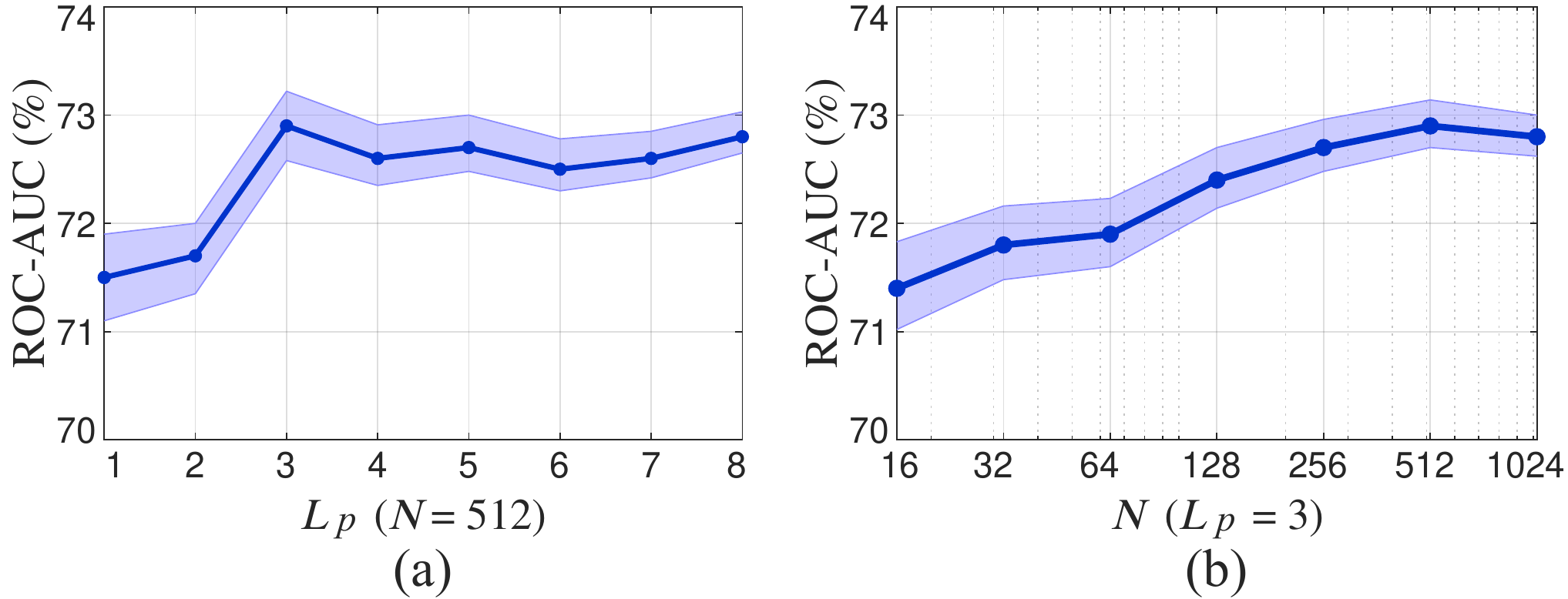}
	\vspace{-4.5mm}
	\caption{Sensitivity analysis of hierarchical prototypes' depth $L_p$ and batch size $N$. (All results are evaluated on biology domain.)} 
	\label{fig_sensitivity}
	\vspace{-2.5mm}
\end{figure}


\textbf{Sensitivity of hierarchical prototypes' depth $L_p$.} In this part, we discuss the selection of parameter $L_p$ which controls the number of discovered semantic hierarchies. In Fig.~\ref{fig_sensitivity}(a), we plot model's performance under different $L_p$ values. It can be observed that deeper hierarchical prototypes (\emph{i.e.} $L_p \geqslant 3$) achieve stable performance gain compared to the shallow ones (\emph{i.e.} $L_p \leqslant 2$). 


\textbf{Sensitivity of batch size $N$.} In this experiment, we evaluate the effect of the batch size $N$ on our method. Fig.~\ref{fig_sensitivity}(b) shows the test ROC-AUC on downstream task using different batch sizes. From the line chart, we can observe that large batch size (\emph{i.e.} $N \geqslant 256$) can promote the performance of GraphLoG. Under such condition, the sampled mini-batches can better represent the whole data set and thus derive more precise likelihood expectation in Eq.~\ref{eq14}. 


\textbf{Visualization.} In Fig.~\ref{fig_tsne}, we utilize the t-SNE~\citep{tsne} to visualize the graph embeddings and hierarchical prototypes on ZINC15 data set. Compared with only using the local constraints $\mathcal{L}_{\textrm{sub}}$ and $\mathcal{L}_{\textrm{graph}}$ (configurations (a) and (b)), more obvious feature separation is achieved after applying the global constraint $\mathcal{L}_{\textrm{global}}$ (configuration (c)), which illustrates its effectiveness on discovering the underlying global-semantic structure of the data. 


\vspace{-1mm}
\section{Conclusions and Future Work} \label{sec6}
\vspace{-0.5mm}

We design a unified framework called Local-instance and Global-semantic Learning (GraphLoG) for self-supervised graph representation learning, which models the structure of a set of unlabeled graphs both locally and globally. In this framework, we novelly propose to learn hierarchical prototypes upon graph embeddings to infer the global-semantic structure in graphs. Using the benchmark data sets from both chemistry and biology domains, we empirically verify our method's superior performance over state-of-the-art approaches on different GNN architectures. 

Our future works will include further improving the global structure learning technique, unifying pre-training and fine-tuning, and extending our framework to other domains such as sociology, physics and material science. 


\section*{Acknowledgements} \label{sec7}

This project was supported by the Natural Sciences and Engineering Research Council (NSERC) Discovery Grant, the
Canada CIFAR AI Chair Program, collaboration grants between Microsoft Research and Mila, Samsung Electronics Co., Ldt., Amazon Faculty Research Award, Tencent AI Lab Rhino-Bird Gift Fund and a NRC Collaborative R\&D Project (AI4D-CORE-08). This project was also partially funded by IVADO Fundamental Research Project grant PRF2019-3583139727. Bingbing Ni is supported by National Science Foundation of China (U20B2072, 61976137).

The authors would also like to thank Meng Qu, Shengchao Liu, Zhaocheng Zhu and Zuobai Zhang for providing constructive advices during this project, and also appreciate the Student Innovation Center of SJTU for providing GPUs.


\bibliography{reference}
\bibliographystyle{icml2021}


\newpage
\appendix
\onecolumn


\section{Theoretical Analysis} \label{supp_sec1}

\begin{mytheorem} \label{theorem1}
Given a mini-batch $\widetilde{\mathbf{G}}$ (batch size is $N$) randomly sampled from the data set $\mathbf{G}$ which contains $M$ graphs, the expected log-likelihood defined on this mini-batch, \emph{i.e.} $\widetilde{Q}(\theta, \mathbf{C}) = \mathbb{E}_{p(\widetilde{\mathbf{Z}} | \widetilde{\mathbf{G}}, \theta_{t-1}, \mathbf{C}_{t-1})} [\log p(\widetilde{\mathbf{G}}, \widetilde{\mathbf{Z}} | \theta, \mathbf{C})]$, is approximately proportional to the expected complete-data log-likelihood, \emph{i.e.} $Q(\theta, \mathbf{C}) = \mathbb{E}_{p(\mathbf{Z} | \mathbf{G}, \theta_{t-1}, \mathbf{C}_{t-1})} [\log p(\mathbf{G}, \mathbf{Z} | \theta, \mathbf{C})]$:
\begin{displaymath}
\widetilde{Q}(\theta, \mathbf{C}) \approx \frac{N}{M} \, Q(\theta, \mathbf{C}) .
\end{displaymath}
\end{mytheorem}
\begin{proof}
For each graph $\mathcal{G}_n$ in the mini-batch, a latent variable $\hat{z}_{\mathcal{G}_n} \sim p(z_{\mathcal{G}_n} | \mathcal{G}_n, \theta_{t-1}, \mathbf{C}_{t-1})$ is sampled from the posterior distribution for Monte Carlo estimation, and the mini-batch log-likelihood can be estimated as follows:
\begin{displaymath}
\widetilde{Q}(\theta, \mathbf{C}) \approx \sum_{n=1}^{N} \log p(\mathcal{G}_n, \hat{z}_{\mathcal{G}_n} | \theta, \mathbf{C}) .
\end{displaymath}
Since the graphs in both mini-batch $\widetilde{\mathbf{G}}$ and data set $\mathbf{G}$ can be regarded as randomly sampled from the data distribution $P_{\mathcal{G}}$, we deduce as below:
\begin{displaymath}
\begin{split}
\widetilde{Q}(\theta, \mathbf{C}) & \approx N \cdot \frac{1}{N} \sum_{n=1}^{N} \log p(\mathcal{G}_n, \hat{z}_{\mathcal{G}_n} | \theta, \mathbf{C}) \\
& = N \cdot \mathbb{E}_{\mathcal{G} \sim P_{\mathcal{G}}} \big[ \log p(\mathcal{G}, z_{\mathcal{G}} | \theta, \mathbf{C}) \big] \\ 
& = \frac{N}{M} \cdot M \cdot \mathbb{E}_{\mathcal{G} \sim P_{\mathcal{G}}} \big[ \log p(\mathcal{G}, z_{\mathcal{G}} | \theta, \mathbf{C}) \big] \\
& = \frac{N}{M} \cdot \sum_{m=1}^{M} \log p(\mathcal{G}_m, \hat{z}_{\mathcal{G}_m} | \theta, \mathbf{C}) \\
& \approx \frac{N}{M} \cdot \mathbb{E}_{p(\mathbf{Z} | \mathbf{G}, \theta_{t-1}, \mathbf{C}_{t-1})} \big[ \log p(\mathbf{G}, \mathbf{Z} | \theta, \mathbf{C}) \big] \\
& = \frac{N}{M} \, Q(\theta, \mathbf{C}) .
\end{split}
\end{displaymath}
Here, for each graph $\mathcal{G}_m$ in data set $\mathbf{G}$, a latent variable $\hat{z}_{\mathcal{G}_m} \sim p(z_{\mathcal{G}_m} | \mathcal{G}_m, \theta_{t-1}, \mathbf{C}_{t-1})$ is sampled from the posterior distribution for Monte Carlo estimation.

\end{proof}

\begin{myprop} \label{prop2}
For each EM cycle, the model parameters $\theta$ and $\mathbf{C}$ are updated in such a way that increases the marginal likelihood function $p(\mathbf{G} | \theta, \mathbf{C})$, unless a local maximum is reached on the mini-batch log-likelihood function $\widetilde{Q}(\theta, \mathbf{C})$.
\end{myprop}
\begin{proof}
We verify this claim from the perspective of variational inference. For a mini-batch $\widetilde{\mathbf{G}}$, we suppose that $q(\widetilde{\mathbf{Z}})$ is a variational distribution over the true posterior $p(\widetilde{\mathbf{Z}} | \widetilde{\mathbf{G}}, \theta, \mathbf{C})$. For any choice of $q(\widetilde{\mathbf{Z}})$, the following decomposition of the marginal log-likelihood $\log p(\widetilde{\mathbf{G}} | \theta, \mathbf{C})$ holds:
\begin{displaymath}
\log p(\widetilde{\mathbf{G}} | \theta, \mathbf{C}) = \mathcal{L}(q, \theta, \mathbf{C}) + \mathrm{KL}(q || p) ,
\end{displaymath}
\begin{displaymath}
\mathcal{L}(q, \theta, \mathbf{C}) = \mathbb{E}_{q(\widetilde{\mathbf{Z}})} \big[ \log p(\widetilde{\mathbf{G}}, \widetilde{\mathbf{Z}} | \theta, \mathbf{C}) - \log q(\widetilde{\mathbf{Z}}) \big] ,
\end{displaymath}
\begin{displaymath}
\mathrm{KL}(q || p) = \mathbb{E}_{q(\widetilde{\mathbf{Z}})} \bigg[ \log  \bigg( \frac{q(\widetilde{\mathbf{Z}})}{p(\widetilde{\mathbf{Z}} | \widetilde{\mathbf{G}}, \theta, \mathbf{C})} \bigg) \bigg] ,
\end{displaymath}
where $\mathcal{L}(q, \theta, \mathbf{C})$ is the evidence lower bound (ELBO) of marginal log-likelihood function, \emph{i.e.} $\mathcal{L}(q, \theta, \mathbf{C}) \leqslant \log p(\widetilde{\mathbf{G}} | \theta, \mathbf{C})$ (equality holds when $q(\widetilde{\mathbf{Z}}) = p(\widetilde{\mathbf{Z}} | \widetilde{\mathbf{G}}, \theta, \mathbf{C})$). 

In the E-step, we set the variational distribution equal to the posterior distribution with respect to the current model parameters, \emph{i.e.} $q(\widetilde{\mathbf{Z}}) = p(\widetilde{\mathbf{Z}} | \widetilde{\mathbf{G}}, \theta_{t-1}, \mathbf{C}_{t-1})$, such that the KL divergence term vanishes, and the ELBO equals to the marginal log-likelihood $\log p(\widetilde{\mathbf{G}} | \theta_{t-1}, \mathbf{C}_{t-1})$. If we substitute $q(\widetilde{\mathbf{Z}})$ with $p(\widetilde{\mathbf{Z}} | \widetilde{\mathbf{G}}, \theta_{t-1}, \mathbf{C}_{t-1})$ in the ELBO term, we see that, after the E-step, it takes the following form:
\begin{displaymath}
\begin{split}
\mathcal{L}(q, \theta, \mathbf{C}) & = \mathbb{E}_{p(\widetilde{\mathbf{Z}} | \widetilde{\mathbf{G}}, \theta_{t-1}, \mathbf{C}_{t-1})} \big[ \log p(\widetilde{\mathbf{G}}, \widetilde{\mathbf{Z}} | \theta, \mathbf{C}) \big] - \mathbb{E}_{q(\widetilde{\mathbf{Z}})} \big[ \log q(\widetilde{\mathbf{Z}}) \big] \\
& = \widetilde{Q}(\theta, \mathbf{C}) + \mathcal{H} \big( q(\widetilde{\mathbf{Z}}) \big) ,
\end{split}
\end{displaymath}
where $\mathcal{H}$ denotes the entropy function. 

In the M-step, the variational distribution $q(\widetilde{\mathbf{Z}})$ is fixed, and thus the ELBO term equals to the expected mini-batch log-likelihood $\widetilde{Q}(\theta, \mathbf{C})$ plus a constant:
\begin{displaymath}
\mathcal{L}(q, \theta, \mathbf{C}) = \widetilde{Q}(\theta, \mathbf{C}) + \mathrm{const} .
\end{displaymath}
In this step, we seek to maximize $\widetilde{Q}(\theta, \mathbf{C})$ with respect to model parameters $\theta$ and $\mathbf{C}$, which will increase the value of $\mathcal{L}(q, \theta, \mathbf{C})$ unless a local maximum is reached on $\widetilde{Q}(\theta, \mathbf{C})$. Except for the local maximum case, there will be new values of $\theta$ and $\mathbf{C}$, denoted as $\theta_t$ and $\mathbf{C}_t$, which gives out that:
\begin{displaymath}
\mathrm{KL}(q || p) = \mathrm{KL} \big( p(\widetilde{\mathbf{Z}} | \widetilde{\mathbf{G}}, \theta_{t-1}, \mathbf{C}_{t-1}) \, || \, p(\widetilde{\mathbf{Z}} | \widetilde{\mathbf{G}}, \theta_{t}, \mathbf{C}_{t}) \big) > 0 .
\end{displaymath}
Denoting the increase of the ELBO term after the M-step as $\Delta \mathcal{L}(q, \theta, \mathbf{C}) = \mathcal{L}(q, \theta_{t}, \mathbf{C}_{t}) - \mathcal{L}(q, \theta_{t-1}, \mathbf{C}_{t-1}) > 0$, the increase of the marginal log-likelihood satisfies that:
\begin{displaymath}
\begin{split}
\Delta \log p(\widetilde{\mathbf{G}} | \theta, \mathbf{C}) & = \log p(\widetilde{\mathbf{G}} | \theta_{t}, \mathbf{C}_{t}) - \log p(\widetilde{\mathbf{G}} | \theta_{t-1}, \mathbf{C}_{t-1}) \\
& = \mathcal{L}(q, \theta_{t}, \mathbf{C}_{t}) + \mathrm{KL}(q || p) - \mathcal{L}(q, \theta_{t-1}, \mathbf{C}_{t-1}) \\
& > \Delta \mathcal{L}(q, \theta, \mathbf{C}) ,
\end{split}
\end{displaymath}
where the KL term of $\log p(\widetilde{\mathbf{G}} | \theta_{t-1}, \mathbf{C}_{t-1})$ vanishes due to the operation in the E-step. Similar as the deduction in Theorem~\ref{theorem1}, the complete-data log-likelihood $\log p(\mathbf{G} | \theta, \mathbf{C})$ and the mini-batch log-likelihood $\log p(\widetilde{\mathbf{G}} | \theta, \mathbf{C})$ have the following relation:
\begin{displaymath}
\begin{split}
\log p(\mathbf{G} | \theta, \mathbf{C}) & = \sum_{m=1}^{M} \log p(\mathcal{G}_m | \theta, \mathbf{C}) \\
& = M \cdot \mathbb{E}_{\mathcal{G} \sim P_{\mathcal{G}}} \log p(\mathcal{G} | \theta, \mathbf{C}) \\
& = \frac{M}{N} \cdot \sum_{n=1}^{N} \log p(\mathcal{G}_n | \theta, \mathbf{C}) \\
& = \frac{M}{N} \log p(\widetilde{\mathbf{G}} | \theta, \mathbf{C}) .
\end{split}
\end{displaymath}
From this relation, we can derive that:
\begin{displaymath}
\Delta \log p(\mathbf{G} | \theta, \mathbf{C}) = \frac{M}{N} \Delta \log p(\widetilde{\mathbf{G}} | \theta, \mathbf{C}) > \frac{M}{N} \Delta \mathcal{L}(q, \theta, \mathbf{C}) > 0 ,
\end{displaymath}
which illustrates that the EM cycle in our approach is able to increase the complete-data marginal likelihood $p(\mathbf{G} | \theta, \mathbf{C})$ except that a local maximum is reached on $\widetilde{Q}(\theta, \mathbf{C})$. 

\end{proof}


\section{More Implementation Details} \label{supp_sec2}

\textbf{Attribute masking scheme.} For the chemistry domain, given a molecular graph, we randomly mask the attributes of 30\% nodes (\emph{i.e.} atoms) in it to obtain its correlated counterpart. Specifically, we add an extra dimension to the feature of atom type and atom chirality to indicate masked attribute, and the input features of all masked atoms are set to these extra dimensions. 

For the biology domain, given a protein ego-network, we randomly mask the attributes of 30\% edges in it to derive its correlated counterpart. In specific, we use an extra dimension to indicate masked attribute. For an edge to be masked, the weight of its extra dimension is set as 1, and the weights of all other dimensions are set as 0. 

\textbf{GNN architecture.} All the GNNs in our experiments (\emph{i.e.} GCN~\citep{gcn}, GraphSAGE~\citep{graphsage}, GAT~\citep{gat} and GIN~\citep{gin}) are with 5 layers, 300-dimensional hidden units and a mean pooling readout function. In addition, two attention heads are employed in each layer of the GAT model. 


\end{document}